\documentclass{scrartcl}
\usepackage{url}
\usepackage{tikz,pgf}
\usetikzlibrary{fit}
\usepackage{amsmath,amssymb,amsthm}
\usepackage{xfrac}
\usepackage{booktabs}
\usepackage{multirow}
\usepackage{mdframed}
\newcommand{\spn}{\textsc{spn}}
\newcommand{\map}{\textsc{map}}
\newcommand{\np}{\textsc{np}}
\newcommand{\p}{\textsc{p}}

\newcommand{\spns}{{\spn}s}
\newcommand{\pr}{\mathbb{P}}
\newcommand{\set}[1]{\ensuremath{\mathcal{#1}}}

\allowdisplaybreaks[1]

\newtheorem{theorem}{Theorem}
\newtheorem{corollary}{Corollary}
\newtheorem{definition}{Definition}

\title{Approximation Complexity of \\ Maximum A Posteriori Inference \\ in Sum-Product Networks\footnote{A similar version of this manuscript appeared in  \cite{Conaty2017}.}}
\author{
      \textbf{Diarmaid Conaty}\\
      Queen's University Belfast, UK
      \vspace{0.8em}\\
      \textbf{Denis D.\ Mau\'a}\\
      Universidade de S\~ao Paulo, Brazil
      \vspace{0.8em}\\
      \textbf{Cassio P.\ de Campos}\\
      Queen's University Belfast, UK
      }

\date{}

\begin{document}

\maketitle

\begin{abstract}
  We discuss the computational complexity of approximating maximum a posteriori inference in sum-product networks. We first show \np-hardness in trees of height two by a reduction from maximum independent set; this implies non-approximability within a sublinear factor. We show that this is a tight bound, as we can find an approximation within a linear factor in networks of height two. We then show that, in trees of height three, it is \np-hard to approximate the problem within a factor $2^{f(n)}$ for any sublinear function $f$ of the size of the input $n$. Again, this bound is tight, as we prove that the usual max-product algorithm finds (in any network) approximations within factor $2^{c \cdot n}$ for some constant $c < 1$. Last, we present a simple  algorithm, and show that it provably produces solutions at least as good as, and potentially much better than, the max-product algorithm. We empirically analyze the proposed algorithm against max-product using synthetic and realistic networks.
\end{abstract}

\newpage

\section{Introduction}

Finding the maximum probability configuration is a key step of many
solutions to problems in artificial intelligence such as image segmentation \cite{Geman1984}, 3D
image reconstruction \cite{Boykov1998}, natural language processing \cite{Koo2010}, speech recognition \cite{PeharzKMP14}, sentiment analysis \cite{Zirn2011}, protein design \cite{Szeliski2008} and
multicomponent fault diagnosis \cite{Steinder2004}, to name but a few. This problem is often called (full) maximum a posteriori (\map) inference, or most likely explanation (\textsc{mpe}).

Sum-Product Networks (\spns) are a relatively new class of graphical
models that allow marginal inference in linear time in their size
\cite{PoonDomingos2011}.
This is therefore in sharp difference with
other graphical models such as Bayesian networks and Markov Random
Fields that require \#\p-hard effort to produce marginal inferences
\cite{Darwiche2009}.
Intuitively, an \spn\ encodes an arithmetic circuit
whose evaluation produces a marginal inference
\cite{Darwiche2003}.
\spns\ have received increasing popularity in
applications of machine learning due to their ability to represent
complex and highly multidimensional distributions
\cite{PoonDomingos2011,Amer2012,PeharzKMP14,cheng2014language,Nath2016aaai,AmerT16}.

In his PhD thesis, Peharz showed a direct proof of \np-hardness of \map\ in \spns\ by a reduction from maximum satisfiability \cite[Theorem 5.3]{Peharz2015}.\footnote{The original proof was incorrect, as it encoded clauses by products; a corrected proof was provided by the author in an erratum note \cite{Peharz2017}.}
Later, Peharz et al.\ \cite{PeharzGPD2016} noted that \np-hardness can be proved by transforming a Bayesian network with a naive Bayes structure into a distribution-equivalent \spn\ of height two (this is done by adding a sum node to represent the latent root variable and its marginal distribution and product nodes as children to represent the conditional distributions).
As \map\ inference in the former is \np-hard \cite{DeCampos2011}, the result follows.

In this paper, we show a direct proof of \np-hardness of \map\ inference by a reduction from maximum independent set, the problem of deciding whether there is a subset of vertices of a certain size in an undirected graph such that no two vertices in the set are connected.
This new proof is quite simple, and (as with the reduction from naive Bayesian networks) uses a sum-product network of height two.
An advantage of the new proof is that, as a corollary, we obtain the non-approximability of \map\ inference within a sublinear factor in networks of height two.
This is a tight bound, as we show that there is a polynomial-time algorithm that produces approximations within a linear factor in networks of height two.
Note that an \spn\ of height two is equivalent to a finite mixture of tractable distributions (e.g.\ an ensemble of bounded-treewidth Bayesian networks) \cite{Rooshenas2013}.
We prove that, in networks of height three, it is \np-hard to approximate the problem within any  factor $2^{f(n)}$ for any sublinear function $f$ of the input size $n$, even if the \spn\ is a tree. This a tight bound, as we show that the usual max-product algorithm \cite{Darwiche2003,PoonDomingos2011}, which replaces sums with maximizations, finds an approximation within a factor $2^{c\cdot n}$ for some constant $c<1$.
Table~\ref{tab:results} summarizes these results. As far as we are concerned, these are the first results about the complexity of approximating \map\ in \spns.

\begin{table} \centering
\begin{tabular}{ccc}
  \toprule
  \textsc{Height} & \textsc{Lower bound} & \textsc{Upper bound} \\ \midrule
  $\phantom{\geq} 1$      & 1 & 1 \\
  $\phantom{\geq} 2$      & $(m-1)^{\varepsilon}$ & $m-1$ \\
  $\geq 3$      &  $2^{s^{\varepsilon}}$ & $2^s$ \\ \bottomrule
\end{tabular}
\caption{Lower and upper bounds on the approximation threshold for a polynomial-time algorithm: $s$ denotes the size of the instance, $m$ is the number of internal nodes, $\varepsilon$ is a nonnegative number less than 1.}
\label{tab:results}
\end{table}

We also show that a simple modification to the max-product algorithm leads to an algorithm that produces solutions which are never worse and potentially significantly better than the solutions produced by max-product.  We compare the performance of the proposed algorithm against max-product in several structured prediction tasks using both synthetic networks and \spns\ learned from real-world data. The synthetic networks encode instances of maximum independent set problems.
The purpose of these networks is to evaluate the quality of solutions produced by both algorithms on shallow \spns\ which (possibly) encode hard to approximate \map\ problems. Deeper networks are learned using the \textsc{LearnSPN} algorithm by Gens and Domingos \cite{GensD2013}.
The purpose of these experiments is to assess the relative quality of the algorithms on \spns\ from realistic datasets, and their sensitivity to evidence.
The empirical results show that the proposed algorithm often finds significantly better solutions than max-product does, but that this improvement is less pronounced in networks learned from real data. We expect these results to foster research in new approximation algorithms for \map\ in \spns.

Before presenting the complexity results in Section \ref{sec:complexity}, we first review the definition of sum-product networks, and comment on a few selected results from the literature in Section \ref{sec:spns}.
The experiments with the proposed modified algorithm and max-product appear in Section \ref{sec:experiments}.
We conclude the paper with a review of the main results  in Section~\ref{sec:conclusion}.

\section{Sum-Product Networks} \label{sec:spns}

We use capital letters without subscripts to denote random vectors (e.g. $X$), and capital letters with subscripts to denote random variables (e.g., $X_1$). If $X$ is a random vector, we call the set $\set{X}$ composed of the random variables $X_i$ in $X$ its \emph{scope}.
The scope of a function of a random vector is the scope of the respective random vector. In this work, we constrain our discussion to random variables with finite domains.

Poon and Domingos \cite{PoonDomingos2011} originally defined \spns\ as multilinear functions of indicator variables that allow for space and time efficient representation and inference.
In its original definition \spns\ were not constrained to represent valid distributions; this was achieved by imposing properties of consistency and completeness.
This definition more closely resembles Darwiche's arithmetic circuits which represent the \emph{network polynomial} of a Bayesian network \cite{Darwiche2003}, and also allow efficient inference (in the size of the circuit).
See \cite{Choi2017icml} for a recent discussion on the (dis)similarities between arithmetic circuits and sum-product networks.

Later, Gens and Domingos \cite{GensD2013} re-stated \spns\ as complex mixture distributions as follows.
\begin{itemize}
\item Any univariate distribution is an \spn.
\item Any weighted sum of \spns\ with the same scope and nonnegative weights is an \spn.
\item Any product of \spns\ with disjoint scopes is an \spn.
\end{itemize}
This alternative definition (called generalized \spns\ by Peharz \cite{Peharz2015}) implies \emph{decomposability}, a stricter requirement than consistency.
Peharz et al.\ \cite{PeharzTPD2015} showed that any consistent \spn\ over discrete random variables can be transformed in an equivalent decomposable \spn\ with a polynomial increase in size, and that weighted sums can be restricted to the probability simplex without loss of expressivity.
Hence, we assume in the following that \spns\ are \emph{normalized}: the weights of a weighted sum add up to one.
This implies that \spns\ specify (normalized) distributions. A similar result was obtained by Zhao et al.\ \cite{ZhaoMP2015}.
We note that the base of the inductive definition can also be extended to accommodate any class of tractable distributions (e.g., Chow-Liu trees) \cite{Rooshenas14,Vergari15}. For the purposes of this work, however, it suffices to consider only univariate distributions.

An \spn\ is usually represented graphically as a weighted rooted graph
where each internal node is associated with an operation $+$ or
$\times$, and leaves are associated with variables and distributions.
The arcs from a sum node to its children are weighted according to the
corresponding convex combinations.
The remaining arcs have implicitly weight 1.
The height of an \spn\ is defined as the maximum distance, counted as number of arcs, from the root to a leaf of its graphical representation.
Figure~\ref{fig:spn} shows an example of an \spn\
with scope $\{X_1,X_2\}$ and height two. Unit weights are omitted in the figure.
Note that by definition every node represents an \spn\ (hence a distribution) on its own; we refer to nodes and their corresponding \spns\ interchangeably.

\begin{figure} \centering
  \begin{tikzpicture}[very thick]
    \node[draw,circle] (A) at (0,3.5) {$+$};
    \node[draw,circle] (B) at (-3,2) {$\times$};
    \node[draw,circle] (C) at (0,2) {$\times$};
    \node[draw,circle] (D) at (3,2) {$\times$};
    \node[draw,circle,label=below:{$0.4$},inner sep=2pt] (E) at (-1.5,0.5) {$X_1$};
    \node[draw,circle,label=below:{$0.9$},inner sep=2pt] (F) at (4.5,0.5) {$X_1$};
    \node[draw,circle,label=below:{$0.7$},inner sep=2pt] (G) at (-4.5,0.5) {$X_2$};
    \node[draw,circle,label=below:{$0.2$},inner sep=2pt] (H) at (1.5,0.5) {$X_2$};

    \draw (A) edge node[above left] {0.2} (B);
    \draw (A) edge node[left] {0.5} (C);
    \draw (A) edge node[above right] {0.3} (D);

    \foreach \x/\y in {B/E, C/E, D/F, B/G, C/H, D/H} {
      \draw (\x) edge (\y);
    }
\end{tikzpicture}
\caption{A sum-product network over binary variables $X_1$ and $X_2$. Only the probabilities $\pr(X_i=1)$ are shown.}
\label{fig:spn}
\end{figure}
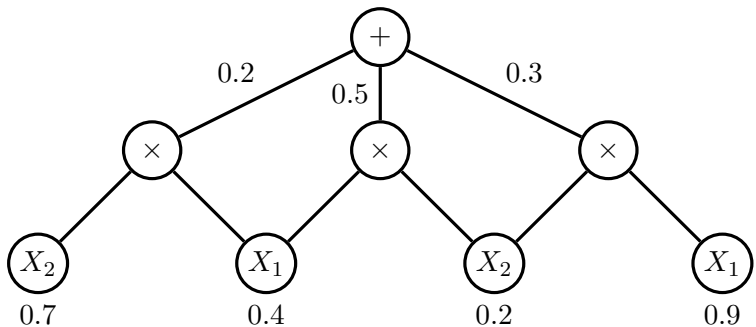

Consider an \spn\ $S(X)$ over a random vector $X=(X_1,\dotsc,X_n)$. The value of $S$ at a point $x=(x_1,\dotsc,x_n)$ in its domain is denoted by $S(x)$ and defined recursively as follows.
The value of a leaf node is the value of its corresponding distribution at the point obtained by projecting $x$ onto the scope of the node.
The value of a product node is the product of the values of its children at $x$. Finally, the value of a sum node is the weighted average of its
children's values at $x$. For example, the value of the \spn\ $S(X_1,X_2)$ in Figure~\ref{fig:spn} at the point $(1,0)$ is $S(1,0)=0.2 \cdot 0.3 \cdot  0.4 +0.5 \cdot 0.4 \cdot 0.8 + 0.3\cdot 0.8 \cdot 0.9 = 0.4$. Note that since we assumed \spns\ to be normalized, we have that $\sum_x S(x) = 1$.

Let $\set{E} \subseteq \{1,\dotsc,n\}$ and consider a random vector
$X_{\set{E}}$ with scope $\{ X_i: i \in \set{E}\}$, and an assignment $e = \{ X_i = e_i: i \in \set{E} \}$. We write $x \sim e$ to denote a value of $X$ consistent with $e$ (i.e., the projection of $x$ on $\set{E}$ is $e$).
Given an \spn\ $S(X)$ representing a distribution $\pr(X)$, we denote the marginal probability $\pr(e)=\sum_{x \sim e} S(x)$ by $S(e)$. This value can be computed by first marginalizing the variables $\{X_j: j \not\in \set{E}\}$ from every (distribution in a) leaf and
then propagating values as before.
Thus marginal probabilities can be
computed in time linear in the network size (considering univariate distributions are represented as tables).
The marginal probability $\pr(X_2=0)=0.7$ induced by the \spn\ in Figure~\ref{fig:spn} can be obtained by first marginalizing leaves without $\{X_2\}$ (thus producing values 1 at respective leaves), and then propagating values as before.

In this work, we are interested in the following computational problem with \spns:

\begin{definition}[Functional \map\ inference problem]
Given an \spn\ $S$ specified with rational weights and an assignment $e$, find $x^*$ such that $S(x^*)=\max_{x \sim e} S(x)$.
\end{definition}

A more general version of the problem would be to allow some of the variables to be summed out, while others are maximized. However, the marginalization (i.e., summing out) of variables can performed in polynomial time as a preprocessing step, the result of which is a \map\ problem as stated above. We stick with the above definition for simplicity (bearing in mind that complexity is not changed).

To prove \np-completeness, we use the decision variant of the problem:

\begin{definition}[Decision \map\ inference problem]
Given an \spn\ $S$ specified with rational weights, an assignment $e$ and a rational $\gamma$, decide whether  $\max_{x \sim e} S(x) \geq \gamma$.
\end{definition}

We denote both problems by \map, as the distinction to which particular (functional or decision) version we refer should be clear from context. Clearly, \np-completeness of the decision version establishes \np-hardness of the functional version. Also, approximation complexity always refers to the functional version.

The support of an \spn\ $S$ is the set of configurations of its domain with positive values: $\text{supp}(S)=\{x: S(x) > 0\}$. An \spn\ is \emph{selective} if for every sub-\spn\ $T$ corresponding to a sum node in $S$ it follows that the supports of any two children are disjoint. Peharz et al.\ \cite{PeharzGPD2016} recently showed that \map\ is tractable in selective \spns.

Here, we discuss the complexity of (approximately) solving \map\ in general \spns. We assume that instances of the \map\ problem are represented as bitstrings $\langle S, e \rangle$ using a reasonable encoding; for instance, weights and probabilities are rational values represented by two integers in binary notation, and graphs are represented by (a binary encoding of their) adjacency lists.

\section{Complexity Results} \label{sec:complexity}

As we show in this section, there is a strong connection between the height of an \spn\ and the complexity of \map\ inferences. First, note that an \spn\ of height 0 is just a marginal distribution. So consider an \spn\ of height 1. If the root is a sum node, then the network encodes a sum of univariate distributions (over the same variable), and \map\ can be solved trivially by enumerating all values of that variable. If on the other hand the root is a product node, then the network encodes a distribution of fully independent variables. Also in this case, we can solve \map\ easily by optimizing independently for each variable. So \map\ in networks of height 1 or less is solvable in polynomial time.

Let us now consider \spns\ of height 2. As already discussed in the introduction, Peharz et al.\ \cite{PeharzGPD2016} briefly observed that the \map\ problem is \np-hard even for tree-shaped networks of height 2. Here, we give the following alternative, direct proof of \np-hardness of \map\ in \spns, that allows us to obtain results on non-approximability.

\begin{theorem} \label{thm:indepset}
  \map\ in sum-product networks is \np-complete even if there is no evidence, and the underlying graph is a tree of height 2.
\end{theorem}
\begin{proof}
Membership is straightforward as we can evaluate the probability of a configuration in polynomial time.

We show hardness by reduction from the \np-hard problem maximum independent set (see e.g.~\cite{Zuckerman07}): Given an undirected graph $G=(V,E)$ with  vertices $\{1,\ldots,n\}$ and an integer $v$, decide whether there is an independent set of size $v$. An independent set is a subset $V' \subseteq V$ such that no two vertices are connected by an edge in $E$.

Let $N_i$ denote the neighbors of $i$ in $V$.
For each $i\in V$, build a product node $S_i$ whose children are leaf nodes $S_{i1},\dotsc,S_{in}$ with scopes $X_1,\dotsc,X_n$, respectively. If $j \in N_i$ then associate $S_{ij}$ with distribution $\pr(X_i=1)=0$; if $j\notin N_i\cup\{i\}$ associate $S_{ij}$ with $\pr(X_j=1)=1/2$; finally, associate $S_{ii}$ with distribution $\pr(X_i=1)=1$.
See Figure~\ref{fig:indset} for an example.
Let $n_i=|N_i|$ be the number of neighbors of $i$. Then $S_i(x)=1/2^{n-n_i-1}$ if  $x_i=1$ and $x_j=0$ for all $j \in N_i$; and $S_i(x)=0$ otherwise. That is, $S_i(x) > 0$ if there is a set $V'$ which contains $i$ and does not contain any of its neighbors.
Now connect all product nodes $S_i$ with a root sum node parent $S$; specify the weight from $S$ to $S_i$ as $w_i=2^{n-n_i-1}/c$, where $c=\sum_i 2^{n-n_i-1}$.
Suppose there is an independent set $I$ of size $v$. Take $x$ such that $x_i=1$ if $i \in I$ and $x_i=0$ otherwise. Then $S(x)=v/c$. For any configuration $x$ of the variables, let $I(x)=\{i: S_i(x) > 0\}$. Then $I(x)$ is an independent set of size $c\cdot S(x)$. So suppose that there is no independent set of size $v$. Then $\max_x S(x) < v/c$. Thus, there is an independent set if and only if $\max_x S(x) \geq v/c$.
\end{proof}

\begin{figure*}[t] \centering
  \resizebox{\textwidth}{!}{
  \begin{tikzpicture}[very thick]
    \begin{scope}
      \node[draw,circle,label={above:$S$}] (S) at (-6,4) {$+$};
      \foreach \x/\y/\z/\w/\o [count=\i] in {%
            $1$/$0$/$0$/$0$/-12cm,%
            $0$/$1$/$0$/$\sfrac{1}{2}$/-8cm,%
            $0$/$0$/$1$/$0$/-4cm,%
            $0$/$\sfrac{1}{2}$/$0$/$1$/0cm
        } {
        \begin{scope}[xshift=\o]
          \node[draw,circle,label={left:$S_\i$}] (P\i) at (0,2) {$\times$};
          \node[draw,circle,label=below:{\x},inner sep=1pt] (X1\i) at (-1.5,0.5) {$X_1$};
          \node[draw,circle,label=below:{\y},inner sep=1pt] (X2\i) at (-0.5,0.5) {$X_2$};
          \node[draw,circle,label=below:{\z},inner sep=1pt] (X3\i) at (0.5,0.5) {$X_3$};
          \node[draw,circle,label=below:{\w},inner sep=1pt] (X4\i) at (1.5,0.5) {$X_4$};
          \foreach \x in {1,2,3,4} {
            \draw (P\i) edge (X\x\i);
          }
        \end{scope}
      }
      \foreach \i/\w in {1/$\sfrac{1}{6}$,2/$\sfrac{1}{3}$,3/$\sfrac{1}{6}$,4/$\sfrac{1}{3}$} {
        \draw (S) edge node[midway, above, sloped, align=center] {\w} (P\i);
      }
    \end{scope}
    \begin{scope}[xshift=4cm,yshift=1.5cm]
      \node[draw,circle] (1) at (0,0) {1};
      \node[draw,circle] (2) at (2,0) {2};
      \node[draw,circle] (3) at (2,2) {3};
      \node[draw,circle] (4) at (0,2) {4};
      \foreach \x/\y in {1/2,2/3,3/4,4/1,1/3} {
        \draw (\x) -- (\y);
      }
    \end{scope}
  \end{tikzpicture}}
  \caption{A sum-product network encoding the maximum independent set problem for the graph on the right. Only the values for $\pr(X_i=1)$ are shown.}
  \label{fig:indset}
\end{figure*}
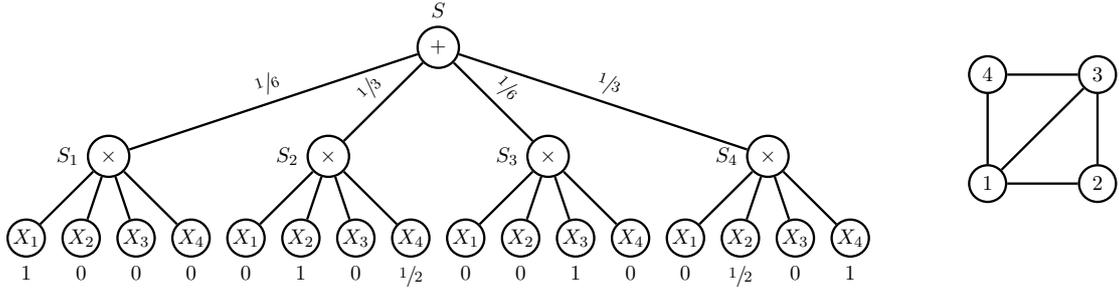

Consider a real-valued function $f(\langle S,e \rangle)$ of the encoded network $S$ and evidence $e$. An algorithm for \map\ in \spns\ is a $f(\langle S,e \rangle)$-approximation if it runs in time polynomial in the size of its input (which specifies the graph, the weights, the distributions, the evidence) and outputs a configuration $\tilde{x}$  such that
$S(\tilde{x}) \cdot f(\langle S,e \rangle) \geq \max_{x\sim e} S(x)$. That is, a $f(\langle S,e \rangle)$-approximation algorithm provides, for every instance $\langle S,e \rangle$ of the \map\ problem, a solution whose value is at most a factor $f(\langle S,e \rangle)$ from the optimum value. The value $f(\langle S,e \rangle)$ is called the \emph{approximation factor}. We have the following consequence of Theorem~\ref{thm:indepset}:

\begin{corollary}\label{cor:indepset}
Unless \p\ equals \np, there is no $(m-1)^{\varepsilon}$-approximation algorithm for \map\ in \spns\ for any $0\leq\varepsilon < 1$, where $m$ is the number of internal nodes of the \spn, even if there is no evidence and the  underlying graph is a tree of height 2.
\end{corollary}

\begin{proof}
The proof of Theorem~\ref{thm:indepset} encodes a maximum independent set problem and the reduction is a weighted reduction (see Definition 1 in~\cite{Bulatov2012}), which suffices to prove the result. To dispense with weighted reductions, we now give a direct proof. So suppose that there is a $(m-1)^{\varepsilon}$-approximation algorithm for \map\ with $0\leq\varepsilon < 1$. Let $\tilde{x}$
be the configuration returned by this algorithm when applied to the \spn\ $S$ created in the proof of Theorem~\ref{thm:indepset} for a graph $G$ given as input of the maximum independent set problem. We have that
\[
S(\tilde{x}) \cdot c \cdot (m-1)^{\varepsilon} \geq c \cdot \max_{x} S(x) = \max_{I \in \set{I}(G)} |I| \, ,
\]
where $c=\sum_i 2^{n-n_i-1}$, $\set{I}(G)$ is the collection of independent sets of $G$, $n$ is the number of vertices in $G$ and $n_i$ is the number of neighbors of vertex $i$ in $G$.
Consequently, this algorithm is a $n^{\varepsilon}$-approximation for maximum independent set (note that $n=m-1$ by construction). We know that there is no $n^{\varepsilon}$-approximation for maximum independent set with $0\leq\varepsilon < 1$ unless \p\ equals \np~\cite{Zuckerman07}, so the result follows.
\end{proof}

\begin{figure*}[t] \centering
  \resizebox{\textwidth}{!}{
  \begin{tikzpicture}[very thick]
    \node[draw,circle] (S) at (0,4) {$+$};
    \foreach \x/\y/\z/\w/\o [count=\i] in {%
          $0$/$0$/$1$/$\sfrac{1}{2}$/-12cm,%
          $1$/$1$/$1$/$\sfrac{1}{2}$/-8cm,%
          $1$/$0$/$0$/$\sfrac{1}{2}$/-4cm,%
          $0$/$1$/$1$/$\sfrac{1}{2}$/0cm,%
          $0$/$0$/$0$/$\sfrac{1}{2}$/4cm,%
          $1$/$1$/$0$/$\sfrac{1}{2}$/8cm,%
          $0$/$1$/$0$/$\sfrac{1}{2}$/12cm%
      } {
      \begin{scope}[xshift=\o]
        \node[draw,circle] (P1\i) at (0,2) {$\times$};
        \node[draw,circle,label=below:{\x},inner sep=1pt] (X1\i) at (-1.5,0.5) {$X_1$};
        \node[draw,circle,label=below:{\y},inner sep=1pt] (X2\i) at (-0.5,0.5) {$X_2$};
        \node[draw,circle,label=below:{\z},inner sep=1pt] (X3\i) at (0.5,0.5) {$X_3$};
        \node[draw,circle,label=below:{\w},inner sep=1pt] (X4\i) at (1.5,0.5) {$X_4$};
        \foreach \x in {1,2,3,4} {
          \draw (P1\i) edge (X\x\i);
        }
      \end{scope}
    }
    \foreach \x/\y/\z/\w/\o [count=\i] in {%
          $0$/$\sfrac{1}{2}$/$0$/$0$/-12cm,%
          $1$/$\sfrac{1}{2}$/$1$/$0$/-8cm,%
          $1$/$\sfrac{1}{2}$/$0$/$1$/-4cm,%
          $0$/$\sfrac{1}{2}$/$1$/$0$/0cm,%
          $0$/$\sfrac{1}{2}$/$0$/$1$/4cm,%
          $1$/$\sfrac{1}{2}$/$1$/$1$/8cm,%
          $0$/$\sfrac{1}{2}$/$1$/$1$/12cm%
      } {
      \begin{scope}[xshift=\o]
        \node[draw,circle] (P2\i) at (0,6) {$\times$};
        \node[draw,circle,label=above:{\x},inner sep=1pt] (X1\i) at (-1.5,7.5) {$X_1$};
        \node[draw,circle,label=above:{\y},inner sep=1pt] (X2\i) at (-0.5,7.5) {$X_2$};
        \node[draw,circle,label=above:{\z},inner sep=1pt] (X3\i) at (0.5,7.5) {$X_3$};
        \node[draw,circle,label=above:{\w},inner sep=1pt] (X4\i) at (1.5,7.5) {$X_4$};
        \foreach \x in {1,2,3,4} {
          \draw (P2\i) edge (X\x\i);
        }
      \end{scope}
    }
    \foreach \p in {1,2,3,4,5,6,7} {
      \draw (S) edge (P1\p);
      \draw (S) edge (P2\p);
    }
\end{tikzpicture}}
\caption{A sum-product network encoding the Boolean formula $(\neg X_1 \vee X_2 \vee \neg X_3) \wedge (\neg X_1 \vee X_3 \vee X_4)$. We represent only the probabilities $\pr(X_i=1)$, and omit the uniform weights $\sfrac{1}{14}$ of the root sum node.}
\label{fig:spn-cnf}
\end{figure*}
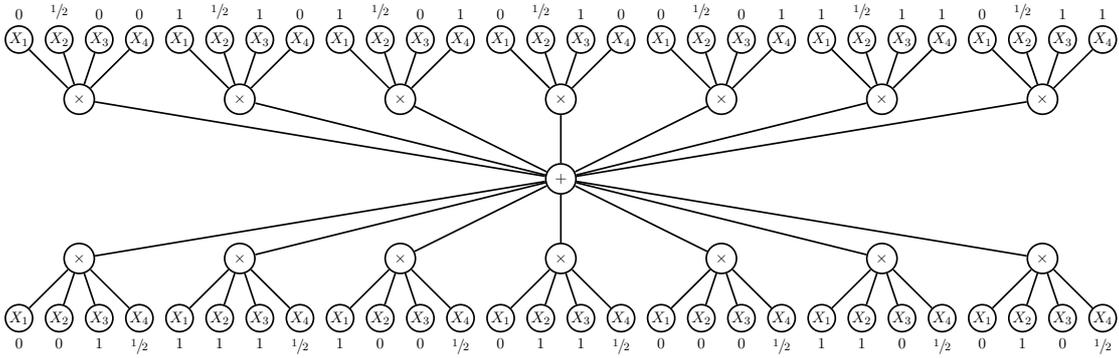

The above result shows that mixtures (or ensembles) of selective \spns\ are  \np-hard to approximate. Note that mixtures of (selective) \spns\ are typically obtained when using bagging to reduce variance in complex models \cite{Breiman1996}. The result can also be used to prove complexity for similar models such as mixture of trees \cite{Meila2000}, mixture of naive Bayes models \cite{Lowd2005} and mixture of arithmetic circuits \cite{Rooshenas2013}.  Corollary~\ref{cor:indepset} can be read as stating that there is \emph{probably} no approximation algorithm for \map\ with sublinear approximation factor in the size of the input. The following result shows that this lower bound is tight:

\begin{theorem} \label{thm:linearapprox}
There exists a $(m-1)$-approximation algorithm for \map\ in sum-product networks whose underlying graph has height at most 2, where $m$ is the number of internal nodes.
\end{theorem}
\begin{proof}
Consider a sum-product network of height 2. If the root is a product node then the problem decomposes into independent \map\ problems in \spns\ of height 1; each of those problems can be solved exactly. So assume that the root $S$ is a sum node connected to either leaf nodes or to nodes which are connected to leaf nodes.
Solve the respective \map\ problem for each child $S_i$ independently (which is exact, as the corresponding \spn\ has height at most 1); denote by $x^i$ the corresponding solution.
Note that $S_i(x^i)$ is an upper bound on the value $S_i(x^*)$, where $x^*$ is a (global) \map\ configuration.
Let $w_1,\dotsc,w_{m-1}$ denote the weights from the root to children $S_1,\dotsc,S_{m-1}$. Return $\tilde{x} = \arg\max_i w_i \cdot S_i(x^i)$. It follows that $(m-1) S(\tilde{x}) \geq \max_{x \sim e} S(x)$. Note that this is the same value returned by the max-product algorithm.
\end{proof}

Thus, for networks of height 2, we have a clear divide: there is an approximation algorithm with linear approximation factor in the number of internal nodes, and no approximation algorithm with sublinear approximation factor in the number of internal nodes. Allowing an additional level of nodes reduces drastically the quality of the approximations in the worst case:

\begin{theorem} \label{thm:noapprox}
 Unless \p\ equals \np, there is no  $2^{s^\varepsilon}$-approximation algorithm for \map\ in \spns\ for any $0\leq\varepsilon< 1$, where $s$ is the size of the input, even if there is no evidence and the underlying graph is a tree of height 3.
\end{theorem}

\begin{proof}
First, we show how to build an \spn\ for deciding satisfiability: Given a Boolean formula $\phi$ in conjunctive normal form, decide if there is a satisfying truth-value assignment. We assume that each clause contains exactly 3 distinct variables (\np-completeness is not altered by this assumption, but if one would like to drop it, then the weights of the sum node we define below could be easily adjusted to account for clauses with less than 3 variables).

Let $X_1,\dotsc,X_n$ denote the Boolean variables and $\phi_1,\dotsc,\phi_m$ denote the clauses in the formula. For $i=1,\dotsc,m$, consider the conjunctions $\phi_{i1},\dotsc,\phi_{i7}$ over the variables of clause $\phi_i$, representing all the satisfying assignments of that clause. For each such assignment, introduce a product node $S_{ij}$ encoding the respective assignment: there is a leaf node with scope $X_k$ whose distribution assigns all mass to value 1 (resp., 0) if and only if $X_k$ appears nonnegated (resp., negated) in $\phi_{ij}$; and there is a leaf node with uniform distribution over $X_k$ if and only if $X_k$ does not appear on $\phi_{ij}$. See Figure~\ref{fig:spn-cnf} for an example.
For a fixed configuration of the random variables, the clause $\phi_i$ is true if and only if one of the product nodes $S_{i1},\dotsc,S_{i7}$ evaluates to 1. And since these products encode disjoint assignments, at most one such product is nonzero for each configuration. We thus have that $\sum_{ij} S_{ij}(x) = m/2^{n-3}$ if $\phi(x)$ is true, and $\sum_{ij} S_{ij}(x) < m/2^{n-3}$ if $\phi(x)$ is false. So introduce a sum node $S$ with all product nodes as children and with uniform weights $1/7m$. There is a satisfying assignment for $\phi$ if and only if $\max_x S(x) \geq 2^{3-n}/7$.

Now take the \spn\ above and make $q$ copies of it with disjoint scopes:
each copy contains different random variables $X_k^t$, $t=1,\dotsc,q$, at the leaves, but otherwise represents the very same distribution/satisfiability problem.
Name each copy $S_t$, $t=1,\ldots q$, and let its size be $s_t$. Denote by $s'=\max_t s_t$ (note that, since they are copies, their size is the same, apart from possible indexing, etc). Connect these copies using a product node $S$ with networks $S_1,\dotsc,S_q$ as children, so that $S(x) = \prod_{t=1}^q S_t(x)$.
Note that $\max_x S_t(x) \geq 2^{3-n}/7$ if there is a satisfying assignment to the Boolean formula, and
$\max_x S_t(x) \leq \frac{m-1}{m}\cdot 2^{3-n}/7$ if there is no satisfying assignment.
Hence,
$\max_x S(x) \geq (2^{3-n}/7)^q$ if there is a satisfying assignment and
$\max_x S(x) \leq ((m-1) 2^{3-n}/(7m))^q$ if there is no satisfying assignment.
Specify
\[
q=1+\left\lfloor(\ln(2)\cdot m \cdot(s'+2)^\varepsilon)^{\frac{1}{1-\varepsilon}}\right\rfloor\, ,
\]
which is polynomial in $s'$, so the \spn\ $S$ can be constructed in
polynomial time and space and has size $s<q(s'+2)$. From the definition of $q$, we have that
\[
q > \left(\ln(2)\cdot m \cdot (s'+2)^\varepsilon\right)^{\frac{1}{1-\varepsilon}} \, .
\]
Raising both sides to $1-\varepsilon$ yields
\[
q > q^\varepsilon \ln(2)\cdot m \cdot (s'+2)^\varepsilon = m  \ln(2^{(q(s'+2))^\varepsilon}) > m \ln 2^{s^\varepsilon} \, .
\]
Since $\frac{1}{m}\leq\ln\frac{m}{m-1}$ for any integer $m>1$, it follows that
\[
q\ln\left(\frac{m}{m-1}\right) > \ln 2^{s^\varepsilon} \, .
\]
By exponentiating both sides, we arrive at
\[
\left(\frac{m}{m-1}\right)^q > 2^{s^\varepsilon}
\,
\text{ hence }
\,
2^{s^\varepsilon}\left(\frac{m-1}{m}\right)^q < 1 \, ,
\]
Finally, by multiplying both sides by $(2^{3-n}/7)^q$, we obtain
\[
2^{s^\varepsilon}\left(\frac{2^{3-n}(m-1)}{7m}\right)^q <  \left(\frac{2^{3-n}}{7}\right)^q \, .
\]
Hence, if we can obtain an $2^{s^\varepsilon}$-approximation for $\max_x S(x)$, then we can decide satisfiability: there is a satisfying assignment to the Boolean formula if and only if the approximation returns a value strictly greater than $(2^{3-n}(m-1)/(7m))^q$.
\end{proof}

According to Theorem~\ref{thm:noapprox}, there is no $2^{f(s)}$-approximation (unless \p\ equals \np) for any sublinear function $f$ of the input size $s$. The following result is used to show that this lower bound is tight.

\begin{theorem} \label{thm:expapprox}
Let $S^+$ denote the sum nodes in \spn\ $S$, and $d_i$ be the number of children of sum node $S_i \in S^+$.
Then there exists a $(\prod_{S_i\in S^+} d_i)$-approximation algorithm for \map\ with input $S$ and evidence $e$.
\end{theorem}

\begin{proof}
There are two cases to consider, based on the value of $S(e)$, which can be checked in polynomial time. If $S(e)=0$, then we can return any assignment consistent with $e$, as the result will be exact (and equal to zero). If $S(e)>0$, then
take the max-product algorithm \cite{PoonDomingos2011}, which consists of an upward pass where sums are replaced by maximizations in the evaluation of an \spn, and a downward pass which selects the maximizers of the previous step.
Define $\text{pd}(S,e)$ recursively as follows.
If $S$ is a leaf then $\text{pd}(S,e)=\max_{x\sim e} S(x)$. If $S$ is a sum node, then $\text{pd}(S,e) = \max_{j=1,\dotsc,t} \, w_j\cdot \text{pd}(S_j,e)$, where $S_1,\dotsc,S_t$ are the children of $S$. Finally, if $S$ is a product node with children $S_1,\dotsc,S_t$, then $\text{pd}(S,e) = \prod_{j=1}^t \text{pd}(S_j,e)$.
Note that $\text{pd}(S,e)$ corresponds to the upward pass of the max-product algorithm; hence it is a lower bound on the value of the configuration obtained by such algorithm.
We prove that the max-product algorithm is a $(\prod_{S_i\in S^+} d_i)$-approximation by proving by induction in the height of the \spn\ that
\[
\text{pd}(S,e) \geq \left(\prod_{S_i\in S^+} \frac{1}{d_i}\right) \max_{x\sim e} S(x) \, .
\]
To show the base of the induction, take a network $S$ of height 0 (i.e., containing a single node). Then $\text{pd}(S,e)=\max_{x\sim e} S(x)$ trivially. So take a network $S$ with children $S_1,\dotsc,S_t$, and suppose (by inductive hypothesis) that
$\text{pd}(S_j,e) \geq (\prod_{S_i\in S^+_j} \frac{1}{d_i}) \max_{x\sim e} S_j(x)$ for every child $S_j$. If $S$ is a product node, then
\begin{align*}
\text{pd}(S,e) = \prod_{j=1}^t \text{pd}(S_j,e) & \geq \prod_{j=1}^t \left(\prod_{S_i\in S^+_j} \frac{1}{d_i}\right) \max_{x\sim e} S_j(x) \\ &= \left(\prod_{S_i\in S^+} \frac{1}{d_i} \right) \prod_{j=1}^t \max_{x\sim e} S_j(x) = \left(\prod_{S_i\in S^+} \frac{1}{d_i}\right)  \max_{x\sim e} S(x)\, ,
\end{align*}
where the last two equalities follow as the scopes of products are disjoints, which implies that the children do not share any node. If $S$ is a sum node, then
\begin{align*}
\text{pd}(S,e) &= \max_{j=1,\dotsc,t}\, w_j\cdot\text{pd}(S_j,e) \\ &\geq \max_{j=1,\dotsc,t}\left(\prod_{S_i\in S^+_j} \frac{1}{d_i}\right) w_j\cdot\max_{x\sim e} S_j(x) \\
&= \max_{j=1,\dotsc,t} \left( \frac{t}{t\prod_{S_i\in S^+_j}d_i}\right) w_j\cdot\max_{x\sim e} S_j(x) \\
&\geq \max_{j=1,\dotsc,t}\, \frac{t}{\prod_{S_i\in S^+} d_i}\cdot w_j\cdot\max_{x\sim e} S_j(x) \\
& = \frac{t\cdot \max_{j=1,\ldots,t} w_j \max_{x\sim e} S_j(x)}{\prod_{S_i\in S^+} d_i}\\
 & \geq \left(\prod_{S_i\in S^+} \frac{1}{d_i}\right) \max_{x\sim e} S(x)\, .
\end{align*}
The first inequality uses the induction hypothesis. The second inequality follows since $1/(t \cdot \prod_{S_i \in S_j^+} d_i) \geq 1/(t \cdot \prod_{S_i \in S^+, S_i \neq S}d_i)=1/\prod_{S_i \in S^+} d_i$.
The last inequality follows as $\max_j w_j \cdot \max_{x\sim e} S_j(x)$ is an upper bound on the value of any child of $S$. This concludes the proof.
\end{proof}

We have the following immediate consequence, showing the tightness of Theorem~\ref{thm:noapprox}.

\begin{corollary}\label{cor:approxmap}
There exists a $2^{\varepsilon \cdot s}$-approximation algorithm for \map\ for some  $0<\varepsilon <1$, where $s$ is the size of the \spn.
\end{corollary}

\begin{proof}
Assume the network has at least one sum node (otherwise we can find an exact solution in polynomial time). Given the result of Theorem~\ref{thm:expapprox}, we only need to show that there is $\varepsilon <  1$ such that $\prod_{S_i\in S^+} d_i < 2^{\varepsilon \cdot s}$,
with $S^+$ the sum nodes in \spn\ $S$ and $d_i$ be the number of children of sum node $S_i \in S^+$.
Because $s$ is strictly greater than the number of nodes and arcs in the network (as we must somehow encode the graph of $S$), we know that $s >  \sum_{S_i\in S^+} d_i$. One can show that $3^{x/3} > x$ for any positive integer.
Hence,
\[
\prod_{S_i\in S^+} d_i \leq \prod_{S_i\in S^+} 3^{d_i/3} = \prod_{S_i\in S^+} 2^{d_i\log_2(3)/3} = 2^{\log_2(3)/3 \cdot\sum_{S_i\in S^+} d_i}   < 2^{s\log_2(3)/3} < 2^{\varepsilon \cdot s} \, ,
\]
for some $\varepsilon < 0.5284$.
\end{proof}

The previous result shows that the max-product algorithm achieves tight upper bounds on the approximation factor. This however does not rule out the existence of approximation algorithms that achieve the same (worst-case) upper bound but perform significantly better on average. For instance, consider the following algorithm that takes an \spn\ $S$ and evidence $e$, and returns $\text{amap}(S,e)$ as follows, where $\text{amap}$ is short for
{\emph{argmax-product}} algorithm.


\begin{mdframed}
\noindent \textbf{Argmax-Product Algorithm}
\begin{itemize}
  \item If $S$ is a sum node with children $S_1,\dotsc,S_t$, then compute
  \[
  \text{amap}(S,e) = \arg\max_{x\in\{x^1,\dotsc,x^t\}} \sum_{j=1}^t w_j\cdot S_j(x) \, , \]
where $x^k = \text{amap}(S_k,e)$,  that is, $x^k$ is the solution of the \map\ problem obtained by argmax-product for network $S_k$ (argmax-product is run bottom-up).
 \item Else if $S$ is a product node with children $S_1,\dotsc,S_t$, then  $\text{amap}(S,e)$ is the concatenation of $\text{amap}(S_1,e)$, $\dotsc$, $\text{amap}(S_t,e)$.
 \item Else, $S$ is a leaf, so return $\text{amap}(S,e)=\arg\max_{x \sim e} S(x)$.
\end{itemize}
\end{mdframed}

Argmax-product has a worst-case time complexity quadratic in the size of the network;
that is because the evaluation of all the children of a sum node with the argument which maximizes each of the children takes linear time (with a smart implementation, it might be possible to achieve subquadratic time).
For comparison, the max-product (with a smart implementation to keep track of solutions and evaluations) takes linear time. While this is a drawback of the argmax-product algorithm, worst-case quadratic time is still quite efficient.
More importantly, argmax-product always produces an approximation at least as good as that of max-product, and possibly exponentially better:

\begin{theorem} \label{thm:amap}
For any \spn\ $S$ and evidence $e$, we have that $S(\text{amap}(S,e)) \geq S(\text{PD}(S,e))$, where $\text{PD}(S,e)$ is the configuration returned by the max-product algorithm. Moreover, there exists $S$ and $e$ such that $S(\text{amap}(S,e)) > 2^{m} S(\text{PD}(S,e))$, where $m$ is the number of sum nodes in $S$.
\end{theorem}

\begin{proof}
It is not difficult to see that $S(\text{amap}(S,e)) \geq S(\text{PD}(S,e))$, because the configuration that is selected by max-product at each sum node is one of the configurations that are tried by the maximization of argmax-product (and both algorithms perform the same operation on leaves and product nodes). To see that this improvement can be exponentially better, consider the \spn\ $S_i$ in Figure~\ref{fig:amap}.
Let $\text{pd}(S_i,e)$ be defined as in the proof of Theorem~\ref{thm:expapprox}.
One can verify that $\text{pd}(S_i,e) = 5/16$, while
\[
 S_i(\text{amap}(S_i,e)) = 3 \cdot 11/48 = 11/16 > 2\cdot 5/16\, .
\]
Now, create an \spn\ $S$ with a product root node connected to children $S_1,\dotsc,S_m$ as described (note that the scope of $S$ is $X_1,\dotsc,X_m$).
Then,
\[
  S(\text{amap}(S,e)) = (11/16)^m > 2^m (5/16)^m = 2^m \cdot \text{pd}(S,e) \, .
\]
The result follows as (for this network) $\text{pd}(S,e)=S(\text{PD}(S,e))$.
\end{proof}

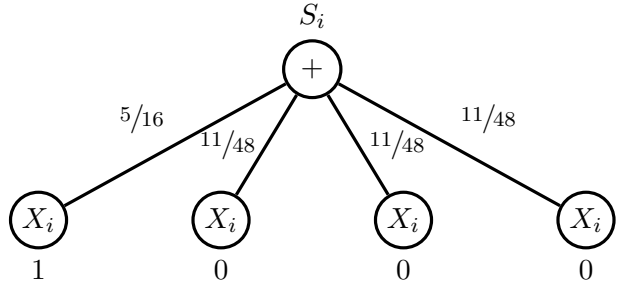
\begin{figure} \centering
  \begin{tikzpicture}[very thick]
    \node[draw,circle,label=above:{$S_i$}] (A) at (0,2.5) {$+$};
    \node[draw,circle,label=below:{$1$},inner sep=2pt] (B) at (-3.6,0.5) {$X_i$};
    \node[draw,circle,label=below:{$0$},inner sep=2pt] (C) at (-1.2,0.5) {$X_i$};
    \node[draw,circle,label=below:{$0$},inner sep=2pt] (D) at (1.2,0.5) {$X_i$};
    \node[draw,circle,label=below:{$0$},inner sep=2pt] (E) at (3.6,0.5) {$X_i$};

    \draw (A) edge node[above left] {$\sfrac{5}{16}$} (B);
    \draw (A) edge node[left] {$\sfrac{11}{48}$} (C);
    \draw (A) edge node[right] {$\sfrac{11}{48}$} (D);
    \draw (A) edge node[above right] {$\sfrac{11}{48}$} (E);

\end{tikzpicture}
\caption{Fragment of the sum-product network used to prove  Theorem~\ref{thm:amap}.}
\label{fig:amap}
\end{figure}

As an immediate result, the solutions produced by argmax-product achieve the  upper bound on the complexity of approximating \map. We hope that this simple result motivates researchers to seek for more sophisticated algorithms that exhibit the time performance of max-product while achieving the accuracy of argmax-product.

\section{Experiments} \label{sec:experiments}

We perform two sets of experiments to verify empirically the difference between argmax-product and max-product. We emphasize that our main motivation is to understand the complexity of the problem and how much it can be approximated efficiently in practice. In the light of Theorem~\ref{thm:amap}, one may suggest that a \map\ problem instance is easy to approximate when argmax-product and max-product produce similar approximations.

In the first set of evaluations, we build \spns\ from random instances for the maximum independent set problem, that is, we generate random undirected graphs with number of vertices given in the first column of Table~\ref{tab:mis} and number of edges presented as percentage of the maximum number of edges (that is, the number of edges in a complete graph). For each graph, we apply the transformation presented in the proof of Theorem~\ref{thm:indepset} to obtain an \spn. The number of nodes of such \spn\ is given in the third column of the table. The fourth column shows the ratio of the values of the configurations found by argmax-product and max-product (that is, $S(\text{amap}(S,e))/S(\text{PD}(S,e))$), averaged over 100 random repetitions.
Take the first row: On average, the result of argmax-product is $1.58$ times better than the result of max-product in \spns\ encoding maximum independent set problems with graphs of 5 vertices and 10\% of edges (which creates \spns\ of 31 nodes: $5\cdot 5=25$ leaves, 5 product nodes and one sum node, same structure as exemplified in Figure~\ref{fig:indset}). The standard deviation for these ratios are also presented (last column in the table). It is clear that argmax-product obtains results that are significantly better than max-product, often surpassing the $2^m=2$ ratio lower bound in Theorem~\ref{thm:amap}.
While in theory argmax-product can be significantly slower than max-product, we did not observe significant differences in running time for these \spns\ with up to 6481 nodes (either algorithm terminated almost instantaneously).

\begin{table}
\centering
\begin{tabular}{ccccc}
\toprule
Vertices & \% Edges & Nodes & Ratio & Std.\ Dev. \\
\midrule
  5 & 10 & 31 & 1.58 & 0.76\\
 5 & 20 & 31 & 1.72 & 0.78\\
 5 & 40 & 31 & 1.61 & 0.75\\
 5 & 60 & 31 & 1.57 & 0.60\\
 10 & 10 & 111 & 1.89 & 0.95\\
 10 & 20 & 111 & 2.16 & 1.09\\
 10 & 40 & 111 & 2.12 & 1.01\\
 10 & 60 & 111 & 2.04 & 0.89\\
 20 & 10 & 421 & 1.94 & 0.88\\
 20 & 20 & 421 & 2.89 & 1.72\\
 20 & 40 & 421 & 3.02 & 1.24\\
 20 & 60 & 421 & 2.60 & 1.04\\
 40 & 10 & 1641 & 2.64 & 1.42\\
 40 & 20 & 1641 & 3.37 & 1.45\\
 40 & 40 & 1641 & 2.33 & 0.73\\
 40 & 60 & 1641 & 2.27 & 0.76\\
 80 & 10 & 6481 & 3.96 & 1.81\\
 80 & 20 & 6481 & 2.10 & 0.49\\
 80 & 40 & 6481 & 1.07 & 0.26\\
 80 & 60 & 6481 & 1.04 & 0.20\\
\bottomrule
\end{tabular}
\caption{Empirical comparison of the quality of approximations produced by  argmax-product and max-product in \spns\ encoding randomly generated instances of the maximum independent set problem.
 Vertices and (percentage of) edges are related to the maximum independent set problem; while nodes is the number of nodes in the corresponding \spn. Each row summarizes 100 repetitions. The last two columns report the ratio of the probabilities of configurations found by argmax-product and max-product, together with its standard deviation.}
\label{tab:mis}
\end{table}

\begin{table*}[th]
\centering
\resizebox{\textwidth}{!}{
\begin{tabular}{lccccccccc}
\toprule
\multirow{2}{*}{Dataset} & No.\ of & No.\ of & No.\ of & No.\ of & No.\ of& \multirow{2}{*}{Height} & No Evidence & \multicolumn{2}{c}{50\% Evidence}\\ \cline{8-10}
 & variables & samples & nodes & nodes $+$ & nodes $\times$ &  & Ratio & Ratio & StDev \\
\midrule
audiology & 70 & 204 & 13513 & 28 & 27 & 12 & 1.0000 & 1.0029 & 0.0133 \\
breast-cancer & 10 & 258 & 1357 & 23 & 24 & 24 & 1.1572 & 1.1977 & 0.1923 \\
car & 7 & 1556 & 16 & 2 & 3 & 3 & 1.1028 & 1.0514 & 0.0514 \\
cylinder-bands & 33 & 487 & 3129 & 61 & 62 & 62 & 1.1154 & 1.1185 & 0.0220 \\
flags & 29 & 175 & 787 & 25 & 26 & 26 & 1.3568 & 1.3654 & 0.0363 \\
ionosphere & 34 & 316 & 603 & 43 & 44 & 42 & 1.1176 & 1.1109 & 0.0273 \\
nursery & 9 & 11665 & 20 & 2 & 3 & 3 & 1.6225 & 1.2060 & 0.2926 \\
primary-tumor & 18 & 306 & 804 & 113 & 114 & 114 & 1.0882 & 1.0828 & 0.0210 \\
sonar & 61 & 188 & 1057 & 101 & 102 & 26 & 1.2380 & 1.2314 & 0.0261 \\
vowel & 14 & 892 & 23 & 2 & 3 & 3 & 1.0751 & 1.0666 & 0.0229 \\
\bottomrule
\end{tabular}}
\caption{Empirical comparison of the quality of approximations produced by  argmax-product and max-product in \spns\ learned from UCI datasets with height at least 2.
The columns indicate the number of variables and samples of each dataset; total number of nodes, sum nodes, product nodes and height of the learned \spn; ratio between argmax-product and max-product solutions for test cases witouth evidence, for test cases with 50\% of variabels as evidence, and the standard deviation of the latter.}
\label{tab:uci}
\end{table*}

In the second set of evaluations, we use realistic \spns\ that model datasets from the UCI Repository.\footnote{Obtained at    \url{http://archive.ics.uci.edu/ml/}.} The \spns\ are learned using an implementation of the ideas presented in~\cite{GensD2013}. For each dataset, we use a random sample of 90\% of the dataset for learning the model and save the remaining 10\% to be used as test.  Then we perform two types of queries. The first type of query consists in using the learned network to compute the mode of the corresponding distribution, that is, running both algorithms with no evidence. In the second type, we split the set of variables in half; for each instance (row) in the test set, we set the respective evidence for the variables in the first half, and compute the \map\ configuration/value for the other half. This allows us to assess the effect of evidence in the approximation complexity.
The results are summarized in Table~\ref{tab:uci}. The first three columns  display information about the dataset (name, number of variables and number of samples); the middle four columns display information about the learned \spn\ (total number of nodes, number of sum nodes, number of product nodes and height); the last three columns show the approximation results: the ratio of probabilities of solutions found by argmax-product and max-product in the test cases with no evidence, the same ratio in test cases with 50\% of variables given as evidence, and the standard deviation for the latter (as it is run over different evidences corresponding to 10\% of the data).
We only show datasets where the learned \spn\ has at least one sum node, since in the other cases the \map\ can be trivially found and a comparison would be pointless. The results suggest that in these \spns\ learned from real data, the difference between argmax-product and max-product is less prominent, yet non negligible. We also see that the complexity of approximation is not considerably affected by the presence of evidence.

\section{Conclusion} \label{sec:conclusion}

We analyzed the complexity of maximum a posteriori inference in sum-product networks and showed that it relates with the height of the underlying graph.
We first provided an alternative (and more direct) proof of \np-hardness of maximum a posteriori inference in sum-product networks.
Our proof uses a reduction from maximum independent set in undirected graphs, from which we obtain the non-approximability for any sublinear factor in the size of input, even in networks of height 2 and no evidence.
We then showed that this limit is tight, that is, that there is a polynomial-time algorithm that produces solutions which are at most a linear factor for networks  of height 2.
We also showed that in networks of height 3 or more, complexity of approximation increases considerably: there is no approximation within a factor $2^{f(n)}$, for any sublinear function $f$ of the input size $n$. This is also a tight bound, as we showed that the usual max-product algorithm finds an approximation within factor $2^{c \cdot n}$ for some constant $c < 1$.
Last, we showed that a simple modification to max-product results in an algorithm that is at least as good, and possibly greatly superior to max-product. We compared both algorithms in two different types of networks: shallow sum-product networks that encode random instances of the maximum independent set problem and deeper sum-product networks learned from real-world datasets. The empirical results show that while the proposed algorithm produces better solutions than max-product does, this improvement is less pronounced in the deeper realistic networks than in the shallower synthetic networks. This suggests that characteristics other than the height of the network might be equally important in determining the hardness of approximating maximum a posteriori inference, and that further (theoretical and empirical) investigations are required. We hope that these results foster research on approximation algorithms for maximum a posteriori inference in sum-product networks.

\section*{Acknowledgements}

Our implementation of the argmax-product algorithm was built on top of the GoSPN library (\url{https://github.com/RenatoGeh/gospn}). We thank Jun Mei for spotting a typo in the proof of Theorem~\ref{thm:noapprox}. The second author received financial support from the S\~ao Paulo Research Foundation (FAPESP) grant \#2016/01055-1 and the CNPq grants \#420669/2016-7 and PQ \#303920/2016-5.

\bibliographystyle{abbrv}

\end{document}